%% file: main.tex
\documentclass{article}
\usepackage{ijcai21}

\usepackage{times}
\usepackage{soul}
\usepackage{url}
\usepackage[hidelinks]{hyperref}
\usepackage[utf8]{inputenc}
\usepackage[small]{caption}
\usepackage{graphicx}
\usepackage{amsmath}
\usepackage{amsthm}
\usepackage{booktabs}
\usepackage{algorithm}
\usepackage{algorithmic}
\usepackage{graphicx}
\usepackage{xcolor}
\usepackage{amsmath}
\usepackage{stmaryrd}
\usepackage{mathrsfs}
\usepackage{amsfonts}
\usepackage{amssymb}
\usepackage{enumerate}
\usepackage{array}
\usepackage{upgreek}
\usepackage{tikz}

\usepackage{cleveref}

\usepackage{breakcites}
\usepackage{etoolbox}
\makeatletter
\patchcmd{\@citex}{,}{,$\!$}{}{}
\makeatother

\newtheorem{example}{Example}
\newtheorem{theorem}{Theorem}
\newtheorem{lemma}[theorem]{Lemma}
\newtheorem{proposition}[theorem]{Proposition}
\newtheorem{corollary}[theorem]{Corollary}
\newtheorem{definition}{Definition}
\newtheorem{observation}{Observation}

\newcommand{\MC}[1]{\mathcal{#1}}

\newcommand{\Mod}[1]{\llbracket#1\rrbracket}
\newcommand{\I}{\omega}

\newcommand{\rel}[1]{\prec_{#1}}
\newcommand{\releq}[1]{\preceq_{#1}}
\newcommand{\relc}[1]{\prec^\circ_{#1}}
\newcommand{\releqc}[1]{\preceq^\circ_{#1}}

\newcommand{\K}{\MC{K}}
\newcommand{\G}{\Gamma}

\newcommand{\relK}{\rel{\K}}
\newcommand{\releqK}{\releq{\K}}
\newcommand{\relcK}{\relc{\K}}
\newcommand{\releqcK}{\releqc{\K}}

\usepackage{scalerel}
\newcommand{\abst}{{\scaleobj{0.75}{(.)}}}
\renewcommand{\land}{\hspace{1pt}{\scaleobj{0.8}{\wedge}}\hspace{1pt}}
\newcommand{\mcirc}{\hspace{1pt}{\scaleobj{0.9}{\circ}}\hspace{1pt}}

\newcommand{\uncovered}{critical loop}

\usepackage{xcolor}

\def\sqsubsetneq{\mathrel{\sqsubseteq\kern-0.92em\raise-0.15em\hbox{\rotatebox{313}{\scalebox{1.1}[0.75]{\(\shortmid\)}}}\scalebox{0.3}[1]{\ }}}

\renewcommand{\models}{\mathrel{\raisebox{-1pt}{$\vDash$}}}

\newcommand{\basechange}{base change}

\title{A General Katsuno-Mendelzon-Style Characterization of\\ 
	AGM Belief Base Revision for Arbitrary Monotonic Logics}

\author{%
	Faiq Miftakhul Falakh$^1$\and
	Sebastian Rudolph$^1$\and
	Kai Sauerwald$^2$ \\
	\affiliations
	$^1$Computational Logic Group, TU Dresden, Germany\\
	$^2$Knowledge Based Systems Group, FernUniversit\"{a}t in Hagen, Germany\\
	\emails
	\{faiq\_miftakhul.falakh, sebastian.rudolph\}@tu-dresden.de,
	kai.sauerwald@fernuni-hagen.de\\[3ex] 
{\begin{flushleft}
\textbf{Important note:}\hspace{-1pt} This article constitutes a preliminary report, which was found to contain inaccuracies. It is\hspace{-0.2pt} superseded by\hspace{-0.2pt} a significantly\hspace{-0.2pt} generalized,\hspace{-1pt} extended,\hspace{-1pt} and revised treatise by\hspace{-0.2pt} the same authors made avai- lable on arXiv.org under the title \emph{Semantic Characterizations of General Belief Base Revision} on the 27th of December 2021 via \url{https://arxiv.org/abs/2112.13557}.
\end{flushleft}
}
}

\begin{document}

\maketitle

\begin{abstract}

The AGM postulates by Alchourr\'{o}n, G\"{a}rdenfors, and Makinson continue to represent a cornerstone in research related to belief change.
We generalize the approach of Katsuno and Mendelzon (KM) for characterizing AGM base revision from propositional logic to the setting of (multiple) base revision in arbitrary monotonic logics. 
Our core result is a representation theorem using the assignment of total -- yet not transitive -- ``preference'' relations to belief bases.
We also provide a characterization of all logics for which our result can be strengthened to preorder assignments (as in KM's original work).

 \end{abstract}

\section{Introduction}\label{sec:introduction}
The question how a rational agent should change her beliefs in the light of new information is crucial to AI systems. It gave rise to the area of \emph{belief change}, which has been massively influenced by the AGM paradigm of Alchourr\'{o}n, G\"{a}rdenfors, and Makinson \shortcite{agm_1985}.
The AGM theory assumes that an agent's beliefs are represented by a deductively closed set of formulas (aka belief set). A change operator for belief sets is required to satisfy appropriate postulates in order to qualify as a rational change operator. 
While the contribution of AGM is widely accepted as solid and inspiring foundation, it lacks support for certain relevant aspects: 
it provides no immediate solution on how to deal with multiple inputs (i.e., several formulae instead of just one), with \textsl{bases} (i.e., arbitrary finite collections of formulae, not necessarily deductively closed), or with the problem of iterated belief changes. 

While the AGM paradigm is axiomatic, much of its success originated from operationalisations via representation theorems. Yet, most existing characterisations of AGM revision require the underlying logic to fulfil the AGM assumptions, including compactness, closure under standard connectives, deduction, and supra-classicality \cite{KS_RibeiroWassermannFlourisAntoniou2013}. %

Leaving the safe grounds of these assumptions complicates matters;
representation theorems do not easily generalize to arbitrary monotonic logics. 
This has sparked investigations into tailored characterisations of AGM belief change for specific logics, such as Horn logic \cite{KS_DelgrandePeppas2015}, temporal logics \cite{KS_Bonanno2007}, action logics \cite{KS_ShapiroPagnuccoLesperanceLevesque2011}, first-order logic \cite{KS_ZhuangWangWangDelgrande2019}, and description logics \cite{qi_knowledge_2006,halaschek-wiener_2006,dong_2017}. 
More general approaches to revision in non-classical logics were given by Ribeiro, Wassermann et al. \cite{KS_RibeiroWassermannFlourisAntoniou2013,KS_Ribeiro2013,KS_RibeiroWassermann2014}, Delgrande et al.~\cite{del_2018}, Pardo et al.~\cite{KS_PardoDellundeGodo2009}, or Aiguier et al.~\cite{aiguier_2018}. 

In this paper, we consider (multiple) revision of finite bases in arbitrary monotonic logics, refining and generalizing the popular approach by Katsuno and Mendelzon {\cite{kat_1991}} (KM) for propositional belief base revision. 
KM start out from finite belief bases, assigning to each a total preorder on the interpretations, which expresses -- intuitively speaking -- a degree of ``modelishness''. The models of the result of any AGM revision will then coincide with the preferred (i.e., preorder-minimal) models of the received information.

Our approach generalises this idea of preferences over interpretations to the general setting, which necessitates adjusting the nature of the ``modelishness-indicating'' assignments: transitivity needs to be waived, whereas certain natural requirements regarding minimality need to be imposed.

The main contributions of this paper are the following:

\smallskip

$\left.\right.$\hspace{-5ex}\begin{minipage}[t]{0.5\textwidth}
\begin{itemize}\setlength{\itemsep}{0pt}%
		\setlength{\parindent}{0ex}%
		\item We extend KM's semantic approach from the setting of singular revision in propositional logic to  multiple revision of finite bases in arbitrary monotone logics.
		\item For this setting, we provide a representation theorem characterizing AGM belief change operators via assignments.
		\item We characterize those logics for which every AGM operator can even be captured by preorder assignments (i.e., in the classical KM way). In particular, this condition applies to all logics supporting disjunction over sentences.   
\end{itemize}
\end{minipage}

\section{Preliminaries}\label{sec:prelim}
We consider arbitrary logics $\mathbb{L}$ with monotonic model-theoretic semantics.
Syntactically, such logics are described by a (possibly infinite) set $\MC{L}$ of \emph{sentences}.
A \emph{belief base} $\K$ is then a finite\footnote{The term \textsl{base} is sometimes also used for arbitrary sets \cite{KS_FermeHansson2018}. We follow the mainstream in computer science and assume finite bases.} subset of $\MC{L}$, that is $\K \in \mathcal{P}_\mathrm{fin}(\MC{L})$.
Unlike in other belief revision frameworks, we impose no further requirements on $\MC{L}$ (such as closure under certain operators).

A model theory for $\mathbb{L}$ is defined in the classical way through a (potentially infinite) class ${\Omega}$ of \emph{interpretations} (also called \emph{worlds}) and a binary relation $\models$ between ${\Omega}$ and $\MC{L}$ where $\I \models \varphi$ indicates that $\I$ is a model of $\varphi$. Hence, a logic $\mathbb{L}$ is specified by the triple $(\MC{L},\Omega,\models)$. 
We let $\Mod{\varphi} = \{\I\in {\Omega} \mid \I \models \varphi\}$ denote the set of all models of $\varphi \in \MC{L}$ and obtain the models of a belief base $\K$ via $\Mod{\K} = \bigcap_{\varphi \in \K} \Mod{\varphi}$. A sentence or belief base is \emph{consistent} if it has a model and \emph{inconsistent} otherwise. 
Logical entailment is defined as usual (overloading the symbol "$\models$") via models: for two belief bases $\K$ and $\K'$ we say $\K$ \emph{entails} $\K'$ (written $\K \models \K'$) if $\Mod{\K} \subseteq \Mod{\K'}$. Note that this definition of the semantics enforces that $\mathbb{L}$ is monotonic.\footnote{From here on, when simply speaking of ``logic'', we always assume the classical, monotonic setting described here. From now on, we also assume a logic $\mathbb{L} = (\MC{L},\Omega,\models)$ as given and fixed.}
As usual we write $\K \equiv \K'$ to express 
$\Mod{\K}=\Mod{\K'}$.
A \emph{multiple \basechange{} operator} for $\mathbb{L}$ is a function $\circ: \mathcal{P}_\mathrm{fin}(\MC{L}) \times \mathcal{P}_\mathrm{fin}(\MC{L}) \to \mathcal{P}_\mathrm{fin}(\MC{L})$. 
For convenience, we henceforth drop ``multiple'' and simply speak of \basechange\ operators instead.

We will endow the interpretation space ${\Omega}$ with some structure. 
A binary relation $\preceq$ over ${\Omega}$ is \emph{total} if, for any $\I_1,\I_2 \in {\Omega}$, at least one of  $\I_1\preceq\I_2$ or $\I_2\preceq\I_1$ holds. We write $\I_1\prec\I_2$ for $\I_1 \preceq \I_2$ and $\I_2 \not\preceq \I_1$.
For $\Omega' \subseteq {\Omega}$, %
$\I \in \Omega'$ is called \emph{$\preceq$-minimal in $\Omega'$} if $\I \preceq \I'$
for all $\I'\in\Omega'$.\footnote{If $\preceq$ is total, this definition 
is equivalent to the \emph{absence} of any $\I'' \in \Omega'$ with $\I'' \prec \I$.}
We let $\min(\Omega',\preceq)$ denote the set of $\preceq$-minimal interpretations in $\Omega'$. We call $\preceq$ a \emph{preorder}, if it is transitive and reflexive.

\section{Base Revision in Propositional Logic}\label{sec:km}

A well-known and by now popular %
characterization of base revision has been described by Katsuno and Mendelzon {\cite{kat_1991}} for the special case of propositional logic. %
KM's approach hinges on several properties of propositional logics. 
To start with, any propositional belief base $\K$ can be written as a single propositional formula $\bigwedge_{\alpha\in \K}\alpha$.   
Consequently, in their approach, belief bases are represented by single formulas. They provide the following set of postulates, derived from the AGM revision postulates,  where $ \varphi,\varphi_1,\varphi_2,\alpha$, and $\beta $ are propositional formulae:

\newcommand{\tad}{\hspace{-1pt}}

{
\begin{itemize}\setlength{\itemsep}{0pt}
	\item[]\hspace{-7mm}(KM1)\hspace{1.7mm}$ \varphi \mcirc \alpha \models \alpha$.
	\item[]\hspace{-7mm}(KM2)\hspace{1.7mm}If $\varphi\land\alpha$ is consistent, then $\varphi \mcirc \alpha\equiv \varphi\land\alpha$.
	\item[]\hspace{-7mm}(KM3)\hspace{1.7mm}If $\alpha$ is consistent, then $\varphi \mcirc \alpha$ is consistent.
	\item[]\hspace{-7mm}(KM4)\hspace{1.7mm}If $\varphi_1 \equiv \varphi_2$ and $ \alpha \equiv \beta$, then $ \varphi_1 \mcirc \alpha \equiv \varphi_2 \mcirc \beta $.
	\item[]\hspace{-7mm}(KM5)\hspace{1.7mm}$(\tad\varphi \mcirc \alpha\tad) \land \beta \models \varphi \mcirc (\tad\alpha \land \beta\tad)$.
	\item[]\hspace{-7mm}(KM6)\hspace{1.7mm}If $(\tad\varphi \mcirc \alpha\tad) \land \beta$ is consistent, then \mbox{$\varphi \mcirc (\tad\alpha \land \beta\tad)\,{\models}\,(\tad\varphi \mcirc \alpha\tad) \land \beta$.}
\end{itemize}}

One key contribution of 
KM is to provide an alternative characterization of those propositional base revision operators satisfying (KM1)--(KM6) by model-theoretic means, i.e. through comparisons between propositional interpretations. 
In the following, we present their results in a formulation that facilitates later generalization. 
One central notion for the characterization is the notion of faithful assignment.

\begin{definition}[assignment, faithful]\label{def:faithful}
	An \emph{assignment} (for $\mathbb{L}$) is a function $\releq{\abst} : \mathcal{P}_\mathrm{fin}(\MC{L}) \to \mathcal{P}({\Omega} \times {\Omega})$ that assigns to each belief base $\K$ a total binary relation $\releqK$ over ${\Omega}$.
	An assignment $\releq{\abst}$ is called \emph{faithful} if it satisfies the following conditions:
\begin{itemize}\setlength{\itemsep}{0pt}
	\item[]\hspace{-4.4ex}(F1)~~~If $\I,\I' \models \K$, then $\I \relK \I'$ does not hold.
	\item[]\hspace{-4.4ex}(F2)~~~If $\I\models \K$ and $\I'\not\models \K$, then $\I \relK \I'$. 
	\item[]\hspace{-4.4ex}(F3)~~~If $\K\equiv\K'$, then ${\releqK} = \releq{\MC{K'}}$.
\end{itemize}
An assignment $\releq{\abst}$ is called a \emph{preorder assignment} if $\releqK$ is a preorder for every $\K \in \mathcal{P}_\mathrm{fin}(\MC{L})$. 	
\end{definition}

Intuitively, faithful assignments provide information which of the two interpretations is ``closer to $\K$-modelhood''. 
Consequently, the actual $\K$-models are $\releqK$-minimal. The next definition captures the idea of an assignment adequately representing the behaviour of a revision operator. 

\begin{definition}[compatible]
	A \basechange{} operator $\circ$ is called \emph{compatible} with some assignment $\releq{\abst}$ if it satisfies  
$
	\Mod{\K\circ{\G}} = \min(\Mod{\G},\releqK) \label{eq:KM}
$
for all belief bases $\K$ and ${\G}$.
\end{definition}

With these notions in place, KM's representation result can be smoothly expressed as follows:

\begin{theorem}[Katsuno and Mendelzon \cite{kat_1991}]\label{thm:km1991}
In propositional logic, a \basechange{} operator $\circ$ satisfies (KM1)--(KM6) iff it is compatible with some faithful preorder assignment.
\end{theorem} 
\section{The Approach}\label{sec:appraoch}
In this section, we prepare our main result by transferring KM's concepts from propositional logic to our general setting.
As mentioned, KM's characterization hinges on features of propositional logic that do not generally hold.
So far, attempts to find similarly elegant formulations for less restrictive logics have made good progress to the benefit of the understanding the nature of AGM revision, yet, none of them capture the very general case considered here %
 (cf.~Section~{\ref{sec:related_works}}).    

For our presentation, we use the following straightforward reformulation of (KM1)--(KM6): %
\newcommand{\nin}{\;\!{\in}\;\!}
\newcommand{\neeq}{\;\!{=}\;\!}
\newcommand{\nnsubseteq}{\;\!{\subseteq}\;\!}
\newcommand{\ncup}{\;\!{\cup}\;\!}
\newcommand{\ncap}{\;\!{\cap}\;\!}
\newcommand{\nsetminus}{\;\!{\setminus}\;\!}
\newcommand{\ncirc}{\;\!{\circ}\;\!}
\newcommand{\nneq}{\;{\neq}\;}
\newcommand{\nmodels}{\;\!{\models}\;\!}
\newcommand{\nG}{\G\!}

\begin{itemize}\setlength{\itemsep}{0pt}
	\item[]\hspace{-4.4ex}(G1)~~$\K \ncirc \G \models \G$.
	\item[]\hspace{-4.4ex}(G2)~~If $\Mod{\K\ncup\G} \nneq \emptyset$ then $\K \ncirc \G\equiv \K\ncup \G$.
	\item[]\hspace{-4.4ex}(G3)~~If $\Mod{\G}\nneq\emptyset$ then $\Mod{\K \ncirc \G}\nneq\emptyset$.
	\item[]\hspace{-4.4ex}(G4)~~If $\K_1 \equiv \K_2$ and $\nG_1 \equiv \nG_2$ then $\K_1 \ncirc \nG_1 \equiv \K_2 \ncirc \nG_2$.
	\item[]\hspace{-4.4ex}(G5)~~$(\K \ncirc \nG_1)\ncup \nG_2 \models  \K \ncirc (\nG_1\ncup \nG_2)$.
	\item[]\hspace{-4.4ex}(G6)~~If $\Mod{(\K \ncirc \nG_1)\ncup \nG_2}\nneq\emptyset$ then $ \K \ncirc (\nG_1\ncup \nG_2) \nmodels (\K \ncirc \nG_1)\ncup \nG_2$.
\end{itemize}
This set of postulates was first given by Qi et al.~\cite{qi_knowledge_2006} in the context of belief base revision specifically for Description Logics, yet, the formulation is generic and perfectly suitable for our general setting, too.
We can see that (G1)--(G6) tightly correspond to (KM1)--(KM6), respectively. One advantage of this presentation is that it does not require $\MC{L}$ to support conjunction (while, of course, conjunction on the sentence level is still implicitly supported via set union of bases).

When switching from the setting of propositional to arbitrary logics, two obstacles become apparent.

{\begin{observation}
		Transitivity in the relation, as required in~Theorem \ref{thm:km1991}, is a too strict property for certain logics.
\end{observation}}

In fact, it has been observed before that the incompatibility between transitivity and KM's approach already arises for propositional Horn logic \cite{KS_DelgrandePeppas2015}.
However, for our result, we need to retain totality as well as a new weaker property (which would come for free with transitivity present) defined next.

\begin{definition}[min-retractive]
	A binary relation $\preceq$ over $ {\Omega} $ is called \emph{min-retractive} (for $\mathbb{L}$) if for every $\G \in \mathcal{P}_\mathrm{fin}(\MC{L})$ and $\I',\I \in \Mod{\G}$ with $\I'\preceq\I$ and $\I\in \min(\Mod{\G},\preceq)$  holds  $\I'\in \min(\Mod{\G},\preceq)$.
\end{definition}

In particular, min-retractivity prevents elements lying on a strict cycle being equivalent to minimal elements.

{\begin{observation}
For arbitrary monotonic logics, the minimum from Definition \ref{eq:KM}, required in Theorem~\ref{thm:km1991}, might be empty. %
\end{observation}}

Thus, one missing ingredient when going to the general case is that of \textit{min-completeness}, defined next.

\begin{definition}[min-complete] 
	A binary relation $\preceq$ over $ {\Omega} $ is called \emph{min-complete} (for $\mathbb{L}$) if for every $\G \in \mathcal{P}_\mathrm{fin}(\MC{L})$ with $\Mod{\G} \not= \emptyset$ holds  $\min(\Mod{\G},\preceq) \not= \emptyset$.
\end{definition}

In the special case of $\preceq$ being transitive and total, min-completeness trivially holds whenever ${\Omega}$ is finite (as, e.g., in the case of propositional logic). In the infinite case, however, it might need to be explicitly imposed, as already noted earlier \cite{del_2018} (cf. also the notion of \emph{limit assumption} by Lewis \cite{lewis1973}).
If $\preceq$ is total but not transitive, min-completeness can be violated even in the finite setting through strict cyclic relationships.

We conveniently unite the two properties into one notion. 

\begin{definition}[min-friendly]
A binary relation $\preceq$ over $ {\Omega} $ is called \emph{min-friendly} (for $\mathbb{L}$) if it is both min-retractive and min-complete. An assignment $\releq{\abst} : \mathcal{P}_\mathrm{fin}(\MC{L}) \to \mathcal{P}({\Omega} \times {\Omega})$ is called min-friendly if $\releqK$ is min-friendly for all $\K \in \mathcal{P}_\mathrm{fin}(\MC{L})$.
\end{definition}

\section{The Representation Theorem}\label{sec:representation_theorem}

We are now ready to generalize KM's representation theorem from propositional to arbitrary monotonic logics, by employing the notion of compatible min-friendly faithful
	assignments. 
\begin{theorem}\label{thm:representation_theorem}
	A \basechange{} operator $\circ$ 
	satisfies (G1)--(G6) iff it is compatible with some min-friendly faithful assignment.
\end{theorem}
We show Theorem~\ref{thm:representation_theorem} in three steps. First, we  provide a canonical way of obtaining an assignment for a given revision operator.
Next, we show that our construction indeed yields a min-friendly faithful assignment that is compatible with the revision operator. 
Finally, we show that the notion of min-friendly compatible assignment is adequate to capture the class of base revision operators satisfying (G1)--(G6). 
	
\subsection{From Postulates to Assignments}

Unfortunately, established methods for obtaining a canonical encoding of the revision strategy of $ \circ $, like the elegant one by Darwiche and Pearl \cite{KS_DarwichePearl1997}, do not generalize well beyond propositional logic. We suggest the following construction, which we consider one of this paper's core contributions. 

\begin{definition}\label{def:relation_new}
	Let $\circ$ be a \basechange{} operator 
	and $\K \in \MC{P}_\mathrm{fin}(\MC{L}) $ a belief base. 
	The relation $ \releqcK $ over  ${\Omega} $ is defined by	
	\begin{multline*}
	\I_1 \releqcK \I_2 \text{ iff for all } \G \in \MC{P}_\mathrm{fin}(\MC{L}) \text{ with } \I_1,\I_2\models\G\\
	\text{holds } \I_1 \models \K\circ\G \text{ or } \I_2 \not\models  \K\circ\G. 
	\end{multline*}
Let $\releqc{\abst} : \mathcal{P}_\mathrm{fin}(\MC{L}) \to \mathcal{P}({\Omega} \times {\Omega})$ denote the mapping $\K \mapsto {\releqcK}$. 
\end{definition}

Intuitively, according to the relation $ \releqcK $, an interpretation $ \I_1 $ is ``at least as $\K$-modelish as'' an interpretation $ \I_2 $ if every change either justifies that $ \I_1 $ is more preferred than $ \I_2 $ or the change yields no information about the preference.
	This construction is strong enough for always obtaining a relation that is total and reflexive. 
	\begin{lemma}[totality]\label{lem:totality}
		If $ \circ $ satisfies (G5) and (G6), the relation $ \releqcK $ is total (and hence reflexive) for every $\K \in \MC{P}_\mathrm{fin}(\MC{L}) $. 
	\end{lemma}
	\begin{proof}
		For totality assume the contrary, i.e. there are $ \releqcK $-incomparable $ \I_1$ and $\I_2 $.
		Due to Definition \ref{def:relation_new}, there must exist $ \G_1,\G_2 \in \MC{P}_\mathrm{fin}(\MC{L}) $ with $ \I_1,\I_2\models\G_1 $  and $ \I_1,\I_2\models\G_2 $, such that $ \I_1 \models \K\circ\G_1 $ and $ \I_2 \not\models \K\circ\G_1 $ as well as $ \I_1 \not\models \K\circ\G_2 $ and $ \I_2 \models \K\circ\G_2 $. From (G6) follows $ \I_1,\I_2 \not\models \K\circ(\G_1\cup\G_2) $. This is a contradiction to (G5), which demands $ \I_1,\I_2 \models \K\circ(\G_1\cup\G_2) $. 
		
		Reflexivity follows immediately from totality.
	\end{proof}

Next comes an auxiliary lemma about belief bases and $ \releqcK $. 
\begin{lemma}\label{lem:help}
Let $ \circ $ satisfy (G5) and (G6) and let $ \K \in \mathcal{P}_\mathrm{fin}(\MC{L})$. 
\begin{enumerate}[(a)]\setlength{\itemsep}{0pt}
	\item If $ \I_1 \not\releqcK \I_2 $, then $ \I_2 \relcK \I_1 $ and there exists some $ \G $ with $ \I_1,\I_2\models\G $ as well as $ \I_2 \models \K\circ\G $ and $ \I_1 \not\models \K\circ\G $.
	\item If there is a $ \G $ with $ \I_1,\I_2\models\G $ such that $ \I_1\models \K\circ\G $, then $ \I_1 \releqcK \I_2 $.
	\item If there is a $ \G $ with $ \I_1,\I_2\models\G $ such that $ \I_1\models \K\circ\G $ and $ \I_2\not\models \K\circ\G $, then $ \I_1 \relcK \I_2 $.
\end{enumerate}
\end{lemma}
\begin{proof} For the proofs of all statements, recall that by Lemma~\ref{lem:totality}, the relation $ \releqcK  $ is total.

\noindent
(a)\ \ By totality, we obtain $ \I_2 \releqcK \I_1 $.
		First assume  there is no $ \G $ with $ \I_1,\I_2\models\G $. Then, $ \I_1 \releqcK \I_2 $ by Definition~\ref{def:relation_new}. Contradiction. Hence, there must be a $ \G $ with $ \I_1,\I_2\models\G $.	
		Toward a contradiction suppose that, for each $ \G $ with $ \I_1,\I_2\models\G $, we have $ \I_1 \models \K\circ\G  $ and $ \I_2 \not \models \K\circ\G  $. Then by Definition~\ref{def:relation_new}, we would gain $ \I_1 \releqcK \I_2 $; again a contradiction.

\noindent
(b)\ \ Let $ \G $ and $ \I_1,\I_2 $ be as assumed. For a contradiction, suppose $ \I_1 \not\releqcK \I_2 $. 
			Then, by part (a) above, there is a $ \G' $ with $ \I_1,\I_2\models\G' $, $ \I_1\not\models\K\circ\G' $ and $ \I_2\models\K\circ\G' $.
			Thus $\I_1$ and $\I_2$ ensure consistency of $ (\K\circ\G)\cup \G' $ and $ (\K\circ\G')\cup \G $, respectively. Using (G5) and (G6) we obtain $ \K\circ(\G\cup \G') \equiv (\K\circ\G')\cup \G $ and $ \K\circ(\G\cup \G')\equiv(\K\circ\G)\cup \G' $. A contradiction, because we obtained $ \I_1 \models \K\circ(\G\cup \G') $ and $ \I_1\not\models\K\circ(\G\cup \G') $.

\noindent
(c)\ \ By Definition \ref{def:relation_new}, the existence of $\G$ implies $ \I_2 \not\releqcK \I_1 $. Then, totality yields $ {\I_1 \releqcK \I_2} $ and hence $ {\I_1 \relcK \I_2} $. \qedhere
\end{proof}

\begin{lemma}[compatibility]\label{lem:compatibility}
If $ \circ $ satisfies (G1), (G3), (G5), and (G6), then it is compatible with $\releqc{\abst}$.
\end{lemma}
\begin{proof}
We have to show that $\Mod{\K\circ\G} = \min(\Mod{\G},\releqcK)$.
For any inconsistent $\G$, the statement is straightforward, since, by (G1), $\Mod{\K\circ\G} = \emptyset = \min(\Mod{\G},\releqcK)$. In the following, we assume consistency of $\G$, showing inclusion in both directions. 

($\subseteq$). From consistency of $\G$ and (G3), we have that $\Mod{\K\circ\G}\neq\emptyset$. 
Hence, there exists some $\I\in \Mod{\K\circ\G}$.
Moreover, for any such $\I$, by (G1), $\I\in \Mod{\G}$.
But then, using Lemma~\ref{lem:help}(b), we can conclude $\I\releqcK\I'$ for any $\I'\in \Mod{\G}$. Consequently, any $\I\in \Mod{\K\circ\G}$ also satisfies $\I\in \min(\Mod{\G},\releqcK)$.

($\supseteq$). Let $\I\in \min(\Mod{\G},\releqcK)$. 
By consistency of $\G$ and (G3), there exists an  $\I' \in  \Mod{\K\circ\G}$.
From the ($\subseteq$)-proof follows $\I' \in \min(\Mod{\G},\releqcK)$. 
Then, by (G1) and Lemma~\ref{lem:help}(b), we obtain $\I' \releqcK \I$ from $\I \in \Mod{\G}$ and $\I' \in \Mod{\G}$ and $\I' \in \Mod{\K\circ\G}$. 
From $\I\in \min(\Mod{\G},\releqcK)$ and $\I' \in \Mod{\G}$ follows $\I \releqcK \I'$, therefore, by Definition~\ref{def:relation_new}, $\I,\I'\in \Mod{\G}$ and $\I' \in \Mod{\K\circ\G}$ enforce $\I\in \Mod{\K\circ\G}$. Concluding, we find that every $\I\in \min(\Mod{\G},\releqcK)$ also satisfies $\I\in \Mod{\K\circ\G}$, as desired.
\end{proof}

\begin{lemma}[min-friendliness]\label{lem:minwell}
	If $ \circ $ satisfies (G1), (G3), (G5), and (G6), then $ \releqcK $ is min-friendly for every $\K \in \MC{P}_\mathrm{fin}(\MC{L}) $. 
\end{lemma}
{\begin{proof}
		Observe that min-completeness is a consequence of (G3) and the compatibility of $ \releqc{\abst} $ with $ \circ $ from Lemma \ref{lem:compatibility}.

	For min-retractivity, suppose toward a contradiction that it didn't hold. 
	Because $ \releqcK $ is total, that means there is a belief base $\G$ and interpretations $\I',\I \models \G$ with $\I'\releqcK\I$ and $\I\in {\min(\Mod{\G},\releqcK)}$ but $\I'\not\in {\min(\Mod{\G},\releqcK)}$.
	From Lemma~\ref{lem:compatibility} we obtain $ \I \models \K\circ\G $ and $ \I'\not\models \K\circ\G $.
	Now, applying Lemma~\ref{lem:help}(c)  yields $ \I\relcK \I' $. A contradiction to $\I'\releqcK\I$.
\end{proof}

\begin{lemma}[faithfulness]\label{lem:faithfulness}
	If $ \circ $ satisfies (G2), (G4), (G5), and (G6), the assignment $\releqc{\abst}$ is faithful. 
\end{lemma}
\begin{proof}
We show satisfaction of the three conditions of faithfulness, (F1)--(F3).

(F1). Let $\I,\I'\in \Mod{\K}$. 
 By (G2) we obtain $ \Mod{\K\circ\K}=\Mod{\K} $. Using Lemma~\ref{lem:help} we have $\I' \releqcK \I$.
This implies $\I \not\relcK \I'$.

(F2). Let $\I\in \Mod{\K}$ and $\I'\not\in \Mod{\K}$.
We get $ \Mod{\K\circ\K}=\Mod{\K} $ from (G2). Then Lemma~\ref{lem:help} implies $\I \releqcK \I'$ and $\I' \not\releqcK \I$.

(F3). Let $\K \equiv \K'$ (i.e. $\Mod{\K} = \Mod{\K'}$).
From Definition~\ref{def:relation_new} and (G4) follows $ \releqcK = \releqc{\K'} $, i.e., $ \I_1 \releqcK \I_2 $ iff $ \I_1 \releqc{\K'} \I_2 $. 
\end{proof}

The previous lemmas can finally be put to use to show that the construction of $\releqc{\abst}$ according to Definition~\ref{def:relation_new} yields an assignment with the desired properties.

\begin{proposition}\label{lfa}
If %
$\circ$ 
satisfies (G1)--(G6), then $\releqc{\abst} $ is a min-friendly 
faithful assignment compatible with $\circ$.
\end{proposition}
\begin{proof}
Assume (G1)--(G6) are satisfied by $ \circ $. Then $\releqc{\abst} $ 
is an assignment since every $\releqK$ is total by Lemma~\ref{lem:totality};
it is min-friendly by Lemma~\ref{lem:minwell};
it is faithful by Lemma~\ref{lem:faithfulness}; and
it is compatible with $\circ$ by Lemma~\ref{lem:compatibility}.\qedhere
\end{proof}

\subsection{From Assignments to Postulates}

Now, it remains to show the ``if'' direction of Theorem~\ref{thm:representation_theorem}.

\begin{proposition}\label{lem:lfa2_new}
	If there exists a min-friendly faithful assignment $\releq{\abst}$ compatible with $\circ$, then $\circ$ satisfies (G1)--(G6).
\end{proposition}

\begin{proof}
	Let $\releq{\abst} : \K \mapsto \releqK $ be as described. We now show that $\circ$ satisfies all of (G1)--(G6).
	
	(G1). Let $\I\in \Mod{\K\circ\G}$. Since $\Mod{\K\circ\G} = \min(\Mod{\G},\releqK)$, we have that $\I\in \min(\Mod{\G},\releqK)$. Then, we also have that $\I\in \Mod{\G}$. Thus, we have that $\Mod{\K\circ\G}\subseteq \Mod{\G}$ as desired.
	
	(G2). Assume $\Mod{\K\cup\G} \neq \emptyset$. 
	By faithfulness, this implies $\Mod{\K}  = \min({\Omega},\releqK)$.
	Thus $ \Mod{\K\cup\G}=
	\min(\Mod{\K\cup\G},\releqK) =
	\min(\Mod{\G},\releqK)
	= \Mod{\K\circ\G} $.
	
	(G3). Assume \mbox{$\Mod{\G} \neq \emptyset$}. By min-completeness, we have $\min(\Mod{\G},\releqK)\neq\emptyset$. 
	Since $\Mod{\K\circ\G} = \min(\Mod{\G},\releqK)$ by compatibility, we obtain $\Mod{\K\circ\G} \neq \emptyset$.
	
	(G4). Suppose there exist $\K_1,\K_2,\G_1,\G_2 \in \MC{P}_\mathrm{fin}(\MC{L})$ with $\K_1 \equiv \K_2$ and $\G_1\equiv\G_2 $. Then, $\Mod{\K_1} = \Mod{\K_2}$ and $\Mod{\G_1} = \Mod{\G_2}$. From (F3), we conclude $\releq{\K_1} = \releq{\K_2}$.
	Now suppose that there exists $\I\in \min(\Mod{\G_1},\releq{\K_1})$ (consequently $\I\in \Mod{\K_1\circ{\G_1}}$). Then $\I\in \Mod{\G_1}$ and also $\I\in \Mod{\G_2}$. Therefore, 
	$\I\in \min(\Mod{\G_2},\releq{\K_2})$ (consequently $\I\in \Mod{\K_2\circ\G_2}$). Thus, $\Mod{\K_1\circ\G_1} \subseteq \Mod{\K_2\circ\G_2}$ holds. Inclusion in the other direction follows by symmetry. Therefore, we have $\K_1\circ\G_1 \equiv \K_2\circ{\G_2}$.
	
	(G5). If $(\K\circ\G_1)\cup \G_2$ is inconsistent, the postulate follows trivially. Now assume that $\Mod{\K\circ\G_1} \cap \Mod{\G_2}\neq\emptyset$, i.e., there is an $\I$ with $\I\in \Mod{\K\circ\G_1} \cap \Mod{\G_2}$.
	Since $\Mod{\K\circ\G_1} = \mbox{$\min(\Mod{\G_1},\releqK)$}$ by compatibility, we have that $\I\in \mbox{$\min(\Mod{\G_1},\releqK)$} \cap \Mod{\G_2}$.
	Suppose for a contradiction that $\I\not\in \Mod{\K\circ(\G_1 \cup\G_2)}$. Then,  $\I\not\in \min(\Mod{\G_1\cup\G_2},\releqK)$. 
	This contradicts $\I \in \min(\Mod{\G_1},\releqK)$.
	Therefore  $\I \in \Mod{\K\circ(\G_1 \cup\G_2)}$. Thus $\Mod{\K\circ\G_1} \cap \Mod{\G_2} \subseteq \Mod{\K\circ(\G_1 \cup\G_2)}$ as desired. 
	
	(G6). Let $ (\K\circ\G_1)\cup \G_2 \not= \emptyset$, thus $ \I' \in \Mod{(\K\circ\G_1)\cup \G_2} = \Mod{\K\circ\G_1} \cap \Mod{\G_2}$ for some $\I'$.
	By compatibility, we then obtain $  \I' \in \min( \Mod{\G_1 }, \releqK)$.
	Now consider an arbitrary $\I$ with $\I\in \Mod{\K\circ(\G_1 \cup \G_2)}$. 	
	By compatibility we obtain $ \I \in\min( \Mod{\G_1 \cup \G_2 } ,\releqK )$ and therefore, since $\I' \in \Mod{\G_1} \cap \Mod{\G_2} = \Mod{\G_1 \cup \G_2}$, we can conclude $\I \releqK \I'$. This and $  \I' \in \min( \Mod{\G_1 }, \releqK)$ imply $\I \in \min( \Mod{\G_1 }, \releqK)$ by min-retractivity. Hence every $\I\in \Mod{\K\circ(\G_1 \cup \G_2)}$ satisfies $\I \in \min( \Mod{\G_1 }, \releqK) = \Mod{\K \circ \G_1}$ but also $\I \in \Mod{\G_2}$, whence $\Mod{\K\circ(\G_1 \cup \G_2)} \subseteq \Mod{\K \circ \G_1} \cap \Mod{\G_2} = \Mod{(\K \circ \G_1) \cup \G_2}$ as desired. \qedhere

\end{proof}

The proof of Theorem \ref{thm:representation_theorem} follows from Proposition~\ref{lfa} and~\ref{lem:lfa2_new}.
 
\section{Abstract Representation Theorem}\label{sec:abstraact_rep}

Theorem~\ref{thm:representation_theorem} establishes the correspondence between operators and assignments under the assumption that $\circ$ is known to exist. Toward a full characterization, we provide an additional condition on assignments, capturing operator existence.

A \emph{semantic \basechange{} function} is a mapping $\mathfrak{R}:\MC{P}_\mathrm{fin}(\MC{L}) \times \MC{P}_\mathrm{fin}(\MC{L}) \to \MC{P}({\Omega})$.
A \basechange{} operator $\circ$ is said to \emph{implement} $\mathfrak{R}$ if for all $\K, \G \in \MC{P}_\mathrm{fin}(\MC{L})$ holds $\Mod{\K \circ \G} = \mathfrak{R} (\K,\G)$. An assignment $\releq{\abst}$ is said to \emph{represent} $\mathfrak{R}$ if $\min(\Mod{\G},\releqK)=\mathfrak{R} (\K,\G)$
for all $\K,\G \in \MC{P}_\mathrm{fin}(\MC{L})$.

For the existence of an operator, it will turn out to be essential that any minimal model set of a belief base obtained from an assignment corresponds to some belief base, a property which is formalized by the following notion.
	\begin{definition}[min-expressible]\label{def:minfinite} 
	Given a logic $\mathbb{L} \,{=}\, (\MC{L},{\Omega},\models)$, a binary relation $\preceq$ over $ {\Omega} $ is called \emph{min-expressible} if for each $\G\in \MC{P}_\mathrm{fin}(\MC{L})$ there exists a belief base $ \MC{B}_{\G,\preceq} \in \MC{P}_\mathrm{fin}(\MC{L}) $ such that $ \Mod{\MC{B}_{\G,\preceq}} \,{=}\, \min(\Mod{\G},\preceq)$.
	An assignment $\releq{\abst}$ will be called min-expressible, if for each $\K \in \MC{P}_\mathrm{fin}(\MC{L})$, $\releqK$ is min-expressible. 
	Given a min-expressible assignment $\releq{\abst}$, let $\circ_{\releq{\abst}}$  denote the \basechange{} operator defined by
	$\K \circ_{\releq{\abst}}\! \G \,{=}\, \MC{B}_{\G,\releqK}$.
\end{definition}

We find the following abstract relation between expressibility, assignments and operators.
\begin{theorem}\label{thm:minexpressible}
Let $\mathbb{L}$ be a logic and let $\mathfrak{R}$ be a semantic \basechange{} function for $\mathbb{L}$.
Then 
$\mathfrak{R}$ is implemented by a \basechange{} operator satisfying (G1)--(G6)
iff
$\mathfrak{R}$ is represented by a min-expressible and min-friendly faithful assignment.
\end{theorem}

\begin{proof}
($\Rightarrow$) Let $\circ$ be the corresponding \basechange{} operator. Then, by Proposition~\ref{lfa}, the assignment $\releq{\abst}^\circ$ as given in Definition~\ref{def:relation_new} is min-friendly, faithful, and compatible with $\circ$, thus it represents $\mathfrak{R}$. As for min-expressibility, recall that, by compatibility, $\Mod{\K\circ\G} =$ $\min(\Mod{\G},\releqcK)$ for every $\G$. As $ \K\circ\G $ is a belief base, min-expressibility follows immediately.

($\Leftarrow$) Let $\releq{\abst}$ be the corresponding min-expressible assignment and $\circ_{\releq{\abst}}$ as provided in Definition~\ref{def:minfinite}. By construction, $\circ_{\releq{\abst}}$ is compatible with $\releq{\abst}$ and therefore implements $\mathfrak{R}$. Proposition~\ref{lem:lfa2_new} implies that $\circ_{\releq{\abst}}$ satisfies (G1)--(G6). \qedhere
\end{proof}

Some colleagues argue that revising bases instead of belief sets calls for syntax-dependence and therefore (G4) should be discarded \cite{KS_FermeHansson2018}. 
Without positioning ourselves in this matter, we would like to emphasize that our characterizations from \Cref{thm:representation_theorem} and \Cref{thm:minexpressible} can be easily adjusted to a more syntax-sensitive setting: a careful inspection of the proofs shows that the results remain valid upon dropping (G4) from the postulates and (F3) from the faithfulness definition. 
\section{Total Preorder Representability}\label{sec:logic_capture_hidden_cycles}
We identify those logics for which every revision operator is representable by a total preorder assignment.

\begin{definition}[total preorder representable]
	A \basechange\ operator $ \circ $ is called \emph{total preorder representable} 
if there is a  min-complete faithful preorder
assignment %
compatible with~$ \circ $. %
\end{definition}

The following setting, describing a relationship between belief bases, will turn out to be the one and only reason to prevent total preorder representability.

\begin{definition}[\uncovered]\label{def:uncovered_new}
Let $ \mathbb{L}=(\MC{L},{\Omega},\models)$ be a logic. Three bases $\G_{\!0}, \G_{\!1}, \G_{\!2} \nin \mathcal{P}_\mathrm{fin}(\MC{L})$ form a \emph{\uncovered} for $ \mathbb{L}$ if
there exist $\K,\G'_{\!0},\G'_{\!1},\G'_{\!2} \nin \mathcal{P}_\mathrm{fin}(\MC{L})$ such that%
\begin{itemize}\setlength{\itemsep}{0pt}%
	\item[]\label{def:uncovered_new:condII}\hspace{-4.4ex}(1)~~$ \Mod{\K\cup \G_{\!0}} \neeq \Mod{\K\cup \G_{\!1}} \neeq \Mod{\K\cup \G_{\!2}} \neeq \emptyset$ %
	\item[]\label{def:uncovered_new:condI}\hspace{-4.4ex}(2)~~$ \emptyset \nneq \Mod{\G'_{\!i}} \nnsubseteq {(\Mod{\G_{\!i}}\ncap\Mod{\G_{\!i\oplus 1}})} \nsetminus \Mod{\G_{\!i\oplus 2}} $ with $i\nin \{0,1,2\}$
\\ (where $\oplus$ is addition $\!\!\!\mod 3$)
	\item[]\label{def:uncovered_new:condIII}\hspace{-4.4ex}(3)~~for any $\G\!\nin \mathcal{P}_\mathrm{fin}(\MC{L})$ with $\Mod{\G'_{\!i} \!\ncup\! \G}\!\nneq\!\emptyset$ for all $ 0{\leq} i {\leq} 2 $ exists a $\G'\nin \mathcal{P}_\mathrm{fin}(\MC{L})$ with $\emptyset \nneq \Mod{\G'} \nnsubseteq \Mod{\G} \!\nsetminus\! (\Mod{\G_{\!0}}\ncup \Mod{\G_{\!1}} \ncup \Mod{\G_{\!2}})$. 
\end{itemize}
\end{definition}

We note that \Cref{def:uncovered_new} generalises a known example for non-total preorder representability in Horn logic \cite{KS_DelgrandePeppas2015,del_2018}.

\renewcommand{\sqsubseteq}{\mathrel{\mbox{$\leqslant\hspace{-10.6pt}\raisebox{0.86pt}{$\vartriangleleft$}$}}}
\renewcommand{\sqsubsetneq}{\vartriangleleft}

\begin{proposition}\label{prop:when_tranisitiveIF}
	If $ \mathbb{L}$ exhibits a \uncovered, then there is a \basechange\ operator $ \circ $ for $ \mathbb{L} $ satisfying (G1)--(G6) that is not total preorder representable.
\end{proposition}
\begin{proof}
	Let $ \G_0,\G_1,\G_{2}\in \mathcal{P}_\mathrm{fin}(\MC{L}) $ form a \uncovered\ and let $ \G'_0,\G'_1,\G'_{2}$ and $ \K $ as in Definition \ref{def:uncovered_new}.

\newcommand{\B}{\mathfrak{B}}

Let $ \B $ denote the set of all $ \G' $ guaranteed by Condition~(3) %
from Definition \ref{def:uncovered_new}, i.e. $ \G'\in \B $ if there is some $ \G $ with 
$ \emptyset\neq \Mod{\G'}\subseteq \Mod{\G}\setminus(\Mod{\G_0}\ncup\Mod{\G_1}\ncup\Mod{\G_2}) $
such that $\Mod{\G'_i}\ncap\Mod{\G}\nneq \emptyset $ for all $ i\in\{0,1,2\}$.
Now let $ \B'=\{ \G\in \B \mid \Mod{\G \cup \K}=\emptyset \} $, i.e., all belief bases from $ \B $ that are inconsistent with $ \K $.
Let $ \leqslant$ be an arbitrary linear order on $ \B' $ with respect to which every non-empty subset of $ \B' $ has a minimum.\footnote{Such a $ \leqslant $ exists due to the well-ordering theorem, by courtesy of the \emph{axiom of choice} \cite{KS_Vialar2017}.} 
We now define $\circ$ as follows: for every $\K' \not\equiv \K$ and any $\G$, let $ \K'\circ\G=\K'\cup\G $ if $ \K'\cup\G $ is consistent, otherwise  $ \K'\circ\G=\G $. 
For $\K$ (and any base equivalent to it), we define:
\begin{equation*}
			\K\circ\G = \begin{cases}
				\K \cup \G                                                          & \text{if } \Mod{\K \ncup \G} \neq \emptyset\text{, otherwise}                                                                \\
				\G \cup \G_{\min}^{\B'}                                             & \text{if }
				\Mod{\G' \ncup \G} \neq \emptyset \text{ for some } \G'\in \B',                                               \\ %
				\G \cup \G'_i                                                       & \text{if } \Mod{\G'_{\!i} \ncup \G} \neq \emptyset \text{ and }\Mod{\G'_{\!i\oplus 2} \cup \G} \,{=}\, \emptyset, \text{and} \\
				\G                                                                  & \text{if none of the above applies,}
			\end{cases}
		\end{equation*}
		where $ \G_{\min}^{\B'}\!=\!\min(\{\G' \nin \B' \mid \Mod{\G' \ncup \G} \nneq \emptyset\}, \leqslant) $.
The construction exploits that Condition (2) implies $\Mod{\G'_0\ncup\G'_1\ncup \G'_2}=\emptyset$ and by Condition (3) every base consistent with $ \G'_0 $, $ \G'_1 $ and $ \G'_2 $ has models outside of $ \Mod{\G_0}\ncup\Mod{\G_1}\ncup\Mod{\G_2} $.

We show that $\circ$ satisfies (G1)--(G6) .
For $ \K'\not\equiv\K $ we obtain a full meet revision  which is known to satisfy (G1)--(G6) \cite{KS_Hansson1999}.
Consider the remaining case of $ \K $ (and any equivalent base):

\emph{Postulates (G1)--(G4)}. The satisfaction of (G1)--(G3) follows direction from the construction of $ \circ $.
For (G4) observe that the case distinction above considers only models of $ \G $ when computing $ \K\circ\G $. Thus, for $ \G_1^*\equiv\G_2^* $ we always obtain $ \K\circ\G_1^*\equiv\K\circ\G_2^* $.

\emph{Postulate (G5) and (G6)}. 
Consider two belief bases $ \G^*_1 $ and $ \G^*_2 $.
If $ \G^*_2 $ is inconsistent  with $ \K\circ\G^*_1 $, then we obtain satisfaction of (G5) immediately.
For the remaining case of (G5) and (G6) we assume  $ \K\circ\G^*_1 $ to be consistent with $ \G^*_2 $.
The postulate (G1) implies that $ \G^*_1\cup \G^*_2 $ is consistent.
If $ \G^*_1 $ is consistent with $ \K $, then we obtain consistency of $ \G^*_2 \cup \K $. 
This implies $ \K\circ(\G^*_1\cup\G^*_2)\equiv(\K\circ\G^*_1)\cup \G^*_2 $; yielding satisfaction of (G5) and (G6).
For the case of consistency of $ \G^*_1 $ with some $ \G'\in \B' $, the set $ \G^*_1\cup\G^*_2 $ is also consistent with $ \G' $.
We obtain  $ (\G^*_1)_{\min}^{\B'}=(\G^*_1 \ncup \G^*_2)_{\min}^{\B'} $, hence $ \G^*_2 $ is consistent with $ \G' $ and for every $ \G'' \in \B' $ with $ \Mod{\G'' \ncup \G^*_1 \ncup \G^*_2}\neq \emptyset $ we also have $ \Mod{\G'' \ncup \G^*_1}\neq \emptyset $.
This yields $ \K\circ(\G^*_1\cup\G^*_2)\equiv(\K\circ\G^*_1)\cup \G^*_2 $, establishing (G5) and (G6) for this case.
If $ \G^*_1 $ is consistent with $ \G'_i $ and inconsistent with $ \G'_{i\oplus 2} $, then likewise $ \G^*_1\cup \G^*_2 $ is consistent with $ \G'_i $ and inconsistent with $ \G'_{i\oplus 2} $. 
Again we obtain $ \K\circ(\G^*_1\cup\G^*_2)\equiv(\K\circ\G^*_1)\cup \G^*_2 $.
If none of the conditions above applies to $ \G^*_1 $, then they also not apply to $ \G^*_1\ncup  \G^*_2 $.
From the construction of $ \circ $ we obtain $ \K\circ(\G^*_1\cup\G^*_2)\equiv(\K\circ\G^*_1)\cup \G^*_2 $.

It remains to show that $\circ$ is not total preorder representable. Towards a contradiction suppose there is a min-complete faithful preorder
assignment $ \releq{\abst} $ for $ \circ $.
By construction there are $ \I_i,\I_j\in\Omega $ with $ \I_i\models\K\circ\G_i $ and $ \I_j\not\models\K\circ\G_i $ for $ 0\leq i,j\leq 2 $ and $ i\neq j $. %
Obtain $ \Mod{\K\circ\G_i}=\min(\Mod{\G_i},\releqK) $ from the compatibility with $ \circ $. %
The definition of $ \circ $ yields $ \I_1\in {\min(\Mod{\G_1},\releqK)} $ %
and $ \I_2\models \G_1 $ and $ \I_2\notin {\min(\Mod{\G_1},\releqK)} $.
We obtain thereof the strict relation $ \I_1 \relK \I_2 $. %
The same argument applies to every pair of interpretations $ \I_i,\I_{i\oplus 1} $.
In summary, we get $ \I_0 \relK \I_1 \relK \I_2 \relK \I_0 $, which is impossible for a transitive relation.
\end{proof}

We call pairs of interpretations detached when the \basechange\ operator gives no hint about how to order them.
	\begin{definition}\label{def:detached}
	A pair $ (\I,\I')\in\Omega\times \Omega $ is called \emph{detached from $ \circ $ in $ \K $}, if $ \I,\I'\not \models\K\circ\G$ for all $ \G\in\MC{P}_\mathrm{fin}(\MC{L}) $. %
\end{definition}
Detached pairs will be helpful when proving the missing part of the correspondence between \uncovered\ and  total preorder representability.
In particular, violations of transitivity in $ \releqcK $ from Definition \ref{def:relation_new} always contain a detached pair.
\begin{lemma}\label{lem:loop_detach}
	Assume $ \mathbb{L} $ does not admit a \uncovered\ and $ \circ $ satisfies (G1)--(G6).
If $ \I_0 \releqcK \I_1 $ and $ \I_1 \releqcK \I_2 $ with $ \I_0 \not\releqcK \I_2 $, then $ (\I_0,\I_1) $ or $ (\I_1,\I_2) $ is detached from $ \circ $ in $ \K $.
\end{lemma}
\begin{proof}
	Towards a contradiction, assume a violation of transitivity, where $ (\I_0,\I_1) $ and $ (\I_1,\I_2) $ are not detached from $ \circ $ in $ \K $.	
	By Lemma \ref{lem:totality} the relation $ \relcK $ is total, and thus we have that $ \I_2 \relcK \I_0 $.	
	Then, due to Lemma~\ref{lem:help}, there exist $ \G_0,\G_1\in\MC{P}_\mathrm{fin}(\MC{L}) $ satisfying $\I_0,\I_1\models\G_0$ and $ \I_0 \models\K\circ\G_0 $ as well as $ \I_1,\I_2\models\G_1 $ and $ \I_1 \models\K\circ\G_1$.
	Moreover, there is a $ \G_2 $ with $ \I_0,\I_2\models\G_2 $ such that $ \I_2 \models\K\circ\G_2 $ and $ \I_0\not\models\K\circ\G_2 $. 
	
	Consider Conditions (1) and (2) %
		from Definition \ref{def:uncovered_new}.
	
\textit{Condition (1).} Assume that $ \K $ is consistent with some $ \G_i $. 
By faithfulness and min-retraction we also have $ \I_i\models \K  $. 
By employing (G2) we obtain $ \I_{i\oplus 2} \models \K $ from $ \I_{i\oplus 2} \models \K\circ\G_{i\oplus 2} $.
By an analogue argumentation we obtain $ \I_0,\I_1,\I_2\models \K $, which yields in consequence a contradiction to $ \I_2 \relcK \I_0 $.

	\textit{Condition (2).} 
We show that $ \K\circ\G_i $ is a consistent belief base with $ \Mod{\K\circ\G_i}\nnsubseteq\Mod{\G_i\ncup\G_{i\oplus 1}}\setminus\Mod{\G_{i\oplus 2}} $ for each $ 0\leq i\leq 2 $.
If $ \G_0\ncup\G_1\ncup\G_2 $ is inconsistent, this is immediate.
If $ \G_0\ncup\G_1\ncup\G_2 $ is consistent, then there exists some $ \I_3\models\G_0\ncup\G_1\ncup\G_2 $ and thus, $ \I_i \releqcK \I_3 $ for $ 0\leq i\leq 2 $. 
If $ \I_3 \releqcK \I_i $ for some $ 0\leq i\leq 2 $, then obtain the contradiction $ \I_i\models\K\circ\G_2 $ by employing min-retraction. 
Therefore, we obtain $ \I_i \relcK \I_3 $ and $ \I_3\not\models \K\circ\G_i $ for all $ 0\leq i\leq 2 $.
As consequence, $ \K\circ\G_i $ is a belief base with $ \Mod{\K\circ\G_i}\nnsubseteq\Mod{\G_i\ncup\G_{i\oplus 1}}\setminus\Mod{\G_{i\oplus 2}} $.
Consistency is given by consistency of all $ \G_i $ and (G3).

As $ \G_0,\G_1,\G_2 $ form no \uncovered{} by assumption, yet Con\-ditions (1) and (2) hold, Condition (3) of Definition \ref{def:uncovered_new} must be violated by some $ \G $ with 
$ \Mod{\K{\circ}\G_i}\ncap\Mod{\G}\nneq \emptyset $ for every $ i\in\{0,1,2\} $ such that
for all $ \G' $ with $ \Mod{\Gamma'} {\subseteq} \Mod{\Gamma}\setminus (\Mod{\G_0}\ncup\Mod{\G_1}\ncup \Mod{\G_2})  $ we have $ \Mod{\G'}{=}\emptyset$.
Since $ \Mod{\K\circ\G} $ is a consistent belief base, we obtain $ \Mod{\K\circ\G} \ncap \Mod{\G_i} \neq \emptyset $ for some $ i\in\{0,1,2\} $.
From  min-retractivity  of $ \releqcK $ and $ \Mod{\K\circ\G}=\min(\Mod{\G},\releqcK) $ we obtain  as consequence $ \I_{i}\in{\min(\Mod{\G},\releqcK)} $.
Employing $ \I_i,\I_{i\oplus 2}\in \Mod{\G_{i\oplus 2}} $, min-retractivity and $ \Mod{\G_{i\oplus 2}} \subseteq \Mod{\G} $ we get $ \I_{i\oplus 2}\in{\min(\Mod{\G},\releqcK)}  $. 	
By an iterative application of the same argument we observe the contradiction $ \I_0\releqcK \I_2 $.
\end{proof}

Lemma \ref{lem:loop_detach} allows us to   
 complete the correspondence between \uncovered s and total preorder representability.
\begin{theorem}\label{thm:when_tranisitive} A logic $ \mathbb{L}= {(\MC{L},{\Omega},\models)}$ 
does not admit a \uncovered{} if and only if
every \basechange\ operator for $ \mathbb{L} $ satisfying (G1)--(G6) is total preorder representable. 
\end{theorem}
\begin{proof}
One direction is given by Proposition \ref{prop:when_tranisitiveIF}.
For the other direction, 
assume $ \mathbb{L} $ does not admit a \uncovered{}
and let $ \circ $ be a  \basechange\ operator $ \circ $ satisfying (G1)--(G6).
We will obtain a min-friendly preorder assignment $ \releq{\abst} $ from $ \releqc{\abst} $ in two steps.
Let $ \K $ be an arbitrary belief base $ \K $.
First, let $ \preceq_1 $ be the relation  defined by $ \preceq_1 \,=\, \releqcK\!\! {\setminus} \mathfrak{D} $, where $ \mathfrak{D} $ denotes the set of all $ (\I,\I')\in\Omega\times\Omega $ which are detached from $ \circ $ by $ \K $.
By Lemma \ref{lem:loop_detach}, the relation $ \preceq_1 $ is transitive. 
Since the removal of a detached pair does not influence the minimality of interpretations with respect to a belief base, we have $ {\min(\G,\preceq_1)}={\min(\G,\releqcK)} $ for all  $ \G\in\MC{P}_\mathrm{fin}(\MC{L}) $.
Second, the relation $ \preceq_1 $ can be order extended to a total preorder $ \preceq_2 $ such that $ \preceq_1\subseteq \preceq_2 $ and $ \I_1 \prec_1 \I_2 $ implies $ \I_1 \prec_2 \I_2 $ \cite[Lem. 3]{KS_Hansson1968}.\footnote{Once more, to this end, we have to assume the \emph{axiom of choice}.}
	Choose $ \releqK $ to be exactly $ \preceq_2 $.
The relation $ \releqK $ is a total preorder where the minimality is preserved, i.e. $ {\min(\G,\releqK)}={\min(\G,\releqcK)} $. Moreover, $ \releqK $ inherits min-friendliness from $ \releqcK $.
Thus, $ \circ $ is total preorder representable.  \qedhere
\end{proof}

We close this section with an implication of Theorem \ref{thm:when_tranisitive}.
A logic $\mathbb{L} \,{=}\, (\MC{L},{\Omega},\models)$ is called \emph{disjunctive}, if for every two bases $\G_1,\G_2 \in \MC{P}_\mathrm{fin}(\MC{L}) $ there is a base $ \G_1{\vee}\G_2 \in\MC{P}_\mathrm{fin}(\MC{L})$ such that $ \Mod{\G_1{\vee}\G_2}=\Mod{\G_1}\ncup\Mod{\G_2}$.
This includes the case of any logic allowing for disjunction on the sentence level, i.e., when for every $\gamma,\delta \in \MC{L}$ exists some $\gamma \vee \delta \in \MC{L}$ such that $\Mod{\gamma \vee \delta} = \Mod{\gamma} \cup \Mod{\delta}$, because then  
$\G_1{\vee}\G_2$ can be obtained as $\{\gamma \vee \delta \mid \gamma \in \G_1, \delta \in \G_2\}$.

\begin{corollary}
In a disjunctive logic, every belief change operator satisfying (G1)--(G6) is total preorder representable.
\end{corollary}

\begin{proof}
A disjunctive logic never exhibits a critical loop; Condition (3) would be violated by picking $\G = \G_0 {\vee} \G_1 {\vee} \G_2$. 
\end{proof}

\section{Related Work}\label{sec:related_works}
We are aware of 
two closely related approaches for revising belief bases (or sets) in settings beyond propositional logic, both proposing model-based frameworks for belief revision without fixing a particular logic or the internal structure of interpretations, and characterizing revision operators via minimal models à la KM with some additional assumptions.

Delgrande et al. \shortcite{del_2018}  add additional restrictions both for the   interpretations (aka possible worlds) as well as for the postulates. 
On the interpretation side, unlike us, they restrict their number to be finite. Also they impose a constraint called \emph{regularity} which serves the very same purpose on their preorders as min-expressibility serves on our total relations.
As for the postulates, they extend the basic AGM postulates with a new one, called \emph{(Acyc)}, with the goal to exclude cyclic preference situations (our ``critical loops''). 
Yet, by imposing this postulate, they 
rule out some %
cases of AGM belief revision that we can cover with our framework, which works with (and characterizes) the pristine AGM postulates.

Aiguier et al. \shortcite{aiguier_2018} consider AGM-like belief base revision with possibly infinite sets of interpretations. 
Moreover, like us, they argue in favor of dropping the requirement that assignments have to yield preorders.
However, they rule out (KM4)/(G4) from the postulates, thus immediately restricting attention to the syntax-dependent case. Also, alike Delgrande et al.'s, their characterization imposes an additional postulate.}
On another note, Aiguier et al.~consider some bases, that actually \emph{do} have models, as inconsistent (and thus in need of revision), which in our view is at odds with the foundational assumptions of belief revision.

\section{Conclusion}\label{sec:conclusion}
We presented a characterization of AGM belief base revision in terms of preference assignments, adapting the approach by KM. Contrary to prior work, our result requires no adjustment of the AGM postulates themselves and yet applies to arbitrary monotonic logics with possibly infinite model sets. While we need to allow for non-transitive preference relations, we also precisely identify the logics where the preference relations can be guaranteed to be preorders as in the original KM result. In particular, this holds for all logics featuring disjunction.

As one of the avenues for future work, we will consider iterated revision. 
To this end, our aim is to advance the line of research by Darwiche and Pearl \cite{KS_DarwichePearl1997} to more general logics.

Finally, we will also be working on concrete realisations of the approach presented here in popular KR formalisms such as ontology languages.

\bibliographystyle{abbrv}
\bibliography{bibexport}

\clearpage
\appendix 
\section*{Supplement}

\input{supplement.tex}

\end{document}

%% file: supplement.tex
\maketitle
\setcounter{theorem}{14}

In addition to our paper, we provide in the following some supplementary example and remarks for the reviewer.
We start by presenting a running example.

\begin{example}[based on \cite{del_2018}]\label{ex:logicEX}
	Let $ \mathbb{L}_\mathrm{Ex} = (\MC{L}_\mathrm{Ex},\Omega_\mathrm{Ex},\models_\mathrm{Ex}) $ be the logic defined by $ \MC{L}_\mathrm{Ex}=\{\psi_0,\ldots,\psi_5,\varphi_0, \ldots,\varphi_4 \} $ and $ \Omega_\mathrm{Ex}=\{ \omega_0,\ldots,\omega_5 \} $, with the models relation $ \models_\mathrm{Ex} $ implicitly given by:
	\begin{align*}
		\Mod{\psi_i} & = \{ \omega_i \}  \\
		\Mod{\varphi_0}& = \{ \omega_0,\ldots,\omega_3 \} \\
		\Mod{\varphi_1} & = \{ \omega_1,\omega_2 \} \\
		\Mod{\varphi_2} & = \{ \omega_2,\omega_3 \} \\
		\Mod{\varphi_3} & = \{ \omega_3,\omega_1 \} \\
		\Mod{\varphi_4} & = \{ \omega_1,\ldots,\omega_5 \} 
	\end{align*}
Since defined in the classical model-theoretic way, $ \mathbb{L}_\mathrm{Ex} $ is a monotonic logic.
\end{example}
Note that logic $ \mathbb{L}_\mathrm{Ex} $ has no connectives.
However, this is just for illustrative purposes and we want to highlight that Example \ref{ex:logicEX} is an extension of an example given by Delgrande et. al \cite{del_2018}, which is known to be implementable in Horn logic~\cite{KS_DelgrandePeppas2015}. 

As next step, we provide a \basechange\ operator for $ \mathbb{L}_\mathrm{Ex} $.

\begin{example}[continuation of Example \ref{ex:logicEX}]\label{ex:orderByEx1}
Let $ \K = \{ \psi_0 \} $ and let $ \circ_\mathrm{Ex} $ be the \basechange\ operator defined as follows:
\begin{equation*}
	\K\circ_\mathrm{Ex}\G = \begin{cases}
		\K \ncup \G            & \hspace{-1ex}\text{if } \Mod{\K \ncup \G} \nneq \emptyset\text{, otherwise}                                                                \\
		\G \ncup \{\psi_4\}    & \hspace{-1ex}\text{if } \Mod{\{\psi_4\} \ncup \G} \nneq \emptyset,                                               \\ 
		\G \ncup \{ \psi_1 \}  & \hspace{-1ex}\text{if } \Mod{\{ \psi_1 \} \ncup \G} \nneq \emptyset \text{ and }\Mod{\{\psi_3\} \ncup \G} \,{=}\, \emptyset,  \\
		\G \ncup \{ \psi_2 \}  & \hspace{-1ex}\text{if } \Mod{\{ \psi_2 \} \ncup \G} \nneq \emptyset \text{ and }\Mod{\{\psi_1\} \ncup \G} \,{=}\, \emptyset,  \\
		\G \ncup \{ \psi_3 \}  & \hspace{-1ex}\text{if } \Mod{\{ \psi_3 \} \ncup \G} \nneq \emptyset \text{ and }\Mod{\{\psi_2\} \ncup \G} \,{=}\, \emptyset,  \\
		\G                    & \hspace{-1ex}\text{if none of the above applies,}
	\end{cases}
\end{equation*}
For all  $ \K' $ with $ \K'\equiv\K $ we define $ \K'\circ\G = \K\circ\G $ and for all $ \K' $ with $ \K'\not\equiv\K $ we define
\begin{equation*}
	\K'\circ\G=\begin{cases}
		\K'\cup \G & \text{ if } \K'\cup \G \text{ consistent} \\
		\G & \text{ otherwise.}
	\end{cases}
\end{equation*}
\end{example}

For all $ \K' $ with $ \K'\not\equiv\K $, there is no violation of the postulates (G1)--(G6) since we obtain a full meet revision known to satisfy (G1)--(G6) \cite{KS_Hansson1999}.
It remains to verify satisfaction of (G1)--(G6) for the case of $ \K'\equiv\K $.
Therefore,  we will provide a relation over $ \Omega_\mathrm{Ex} $ and make then use of Theorem \ref{thm:minexpressible} to show satisfaction of (G1)--(G6).
Before continuing with our running example  in the next section, we give some remarks beforehand on the generic technique our paper provides for obtaining such a relation.

\subsection*{On Encoding an Operator into a Preference Relation}

One of the main contributions of our article is a novel way of obtaining a preference relation from an operator. The established way of encoding a revision operator is provided by Katsuno and Mendelzon \cite{kat_1991} as follows:
\begin{equation}\label{eq:km_encoding}
	\I_1 \leq_\K \I_2 \text{ if } \I_1 \models \K \text{ or } \I_1 \models \K\circ \mathit{form}(\I_1,\I_2)
\end{equation}
where $ \mathit{form}(\I_1,\I_2)\in \MC{L} $ denotes a formula with $ \Mod{\mathit{form}(\I_1,\I_2)}=\{ \I_1,\I_2 \} $. 
The problem in a general logical setting is that there might be no such formula $ \mathit{form}(\I_1,\I_2) $ with $ \Mod{\mathit{form}(\I_1,\I_2)}=\{ \I_1,\I_2 \} $ in the considered logic.
This generalises to the non-existence of such belief bases.
In particular, this is the case in our running example (cf. Example \ref{ex:logicEX}).

Clearly, we are not the first to address this problem. 
Delgrande, Peppas and Woltran \cite{del_2018} solve this problem by simultaneously revising with all formulae satisfying $ \I_1 $ and $ \I_2 $, in order to ``simulate'' the revision by the desired formula $ form(\I_1,\I_2) $.
Aiguier, Atif, Block and Hudelot \cite{aiguier_2018} use a similar approach by revising with all formulae at once. 
In summary, neither Aiguier et. al nor Delgrande et al. use a encoding approach fundamentally different to the one by Katsuno and Mendelzon.  

However, in our approach, we consider the restricted setting of revision by finite bases. Therefore, the approach by Aiguier et al. and Delgrande et al. was not feasible for us. 
Moreover, it turns out that applying this idea to the general case would leave the order of certain pairs of elements undetermined.
Depending on the shape of the logic (and its model theory) and the operator, there might be no preference between certain elements (because there is no revision which provides information on the preference). 
We call this pairs of interpretations \emph{detached} (c.f. Definition \ref{def:detached}). In particular, when one wants to obtain a total relation, these elements have to be ordered in a certain way, and selection of a \textit{"preference"} between these two interpretations is a \textit{"non-local"} choice (as it may have ramifications for other ``ordering choices'').

As solution, we came up with Definition \ref{def:relation_new}, which provides an encoding different from the approach by Katsuno and Mendelzon. 
This definition solves the problem with the detached pairs by treating them as equal. 
As a nice resulting property, we obtain that the relation given by Definition \ref{def:relation_new} is a maximal canonical representation for the preferences of an operator -- a property the encoding approaches given by Equation \eqref{eq:km_encoding} do not have.
\begin{proposition}\label{lem:maxrelation}
	Let $ \circ $ be a \basechange\ operator satisfying (G1)--(G6). If  $\releq{\abst}$ is a min-friendly faithful
	assignment compatible with $ \circ $, then  $ \I_1\releqK\I_2  $ implies $ \I_1\releqcK\I_2 $ for every $ \I_1,\I_2\in\Omega $ and every belief base $ \MC{K}\in\MC{P}_\mathrm{fin}(\MC{L}) $.
\end{proposition}
\begin{proof}
	By Theorem \ref{thm:representation_theorem}, the relation $ \releqK $ is min-friendly and total.
	Now assume $ \I_1\releqK\I_2  $ with $ \I_1\not\releqcK\I_2 $. 
	By Lemma~\ref{lem:help}~(a), there is a belief base $ \G $ with $ \I_1,\I_2\models\G $ such that $ \I_2\models\MC{K}\circ\G $ and $ \I_1\not\models\MC{K}\circ\G $.
	Therefore, by compatibility, $ \I_2 \in \min(\Mod{\G},\releq{\K}) = \Mod{\MC{K}\circ\G}$ and $ \I_1 \notin \min(\Mod{\G},\releq{\K}) = \Mod{\MC{K}\circ\G} $,
	a contradiction to $ \I_1\releqK\I_2  $ due to min-retractivity.
\end{proof}

We continue with the running example.

\begin{example}[continuation of Example \ref{ex:orderByEx1}]\label{ex:orderByEx2}
	Applying Definition~\ref{def:relation_new} to $ \K $ and $ \circ_\mathrm{Ex} $ yields the following relation $ \releqK^{\circ_\mathrm{Ex}} $ on $ \Omega_\mathrm{Ex} $ (where $ \omega \relK^{\circ_\mathrm{Ex}} \omega' $ denotes $ \omega\releqK^{\circ_\mathrm{Ex}} \omega' $ and $ \omega' \not\releqK^{\circ_\mathrm{Ex}} \omega $):%
	\begin{align*}
		\omega_i \releqK^{\circ_\mathrm{Ex}} \omega_i & , \ 0\leq i\ \leq 5\\ 
		\omega_0 \relK^{\circ_\mathrm{Ex}} \omega_i & , \ 1\leq i\ \leq 5\\
		\omega_1 \relK^{\circ_\mathrm{Ex}} \omega_2 & \\
		\omega_2 \relK^{\circ_\mathrm{Ex}} \omega_3 & \\
		\omega_3 \relK^{\circ_\mathrm{Ex}} \omega_1 & \\
		\omega_4 \relK^{\circ_\mathrm{Ex}} \omega_i & ,\ i \in \{1,2,3,5\}\\
		\omega_i \relK^{\circ_\mathrm{Ex}} \omega_5 &  , \ 0\leq i\ \leq 4
	\end{align*}
Observe that $ \releqK^{\circ_\mathrm{Ex}} $ is not transitive, since $ \omega_1,\omega_2,\omega_3 $ form a circle. Yet, one can easily verify that $ \releqK^{\circ_\mathrm{Ex}} $ is a total and min-friendly relation.
In particular, as $ \Omega_\mathrm{Ex} $ is finite, min-completeness is directly given.
Moreover, there is no belief base $ \G \in \MC{P}_\mathrm{fin}(\MC{L}_\mathrm{Ex}) $ 
such that there is some $ \omega\notin {\min(\G,\releqK^{\circ_\mathrm{Ex}})} $ and $ \omega' \in \min(\G,\releqK^{\circ_\mathrm{Ex}})  $ with $ \omega \releqK^{\circ_\mathrm{Ex}} \omega' $. 
Note that such a situation could appear in $ \releqK^{\circ_\mathrm{Ex}} $ if a interpretation $ \omega $ would be $ \releqK^{\circ_\mathrm{Ex}} $-equivalent to $ \omega_1$, $ \omega_2$ and $ \omega_3 $ and there would be a belief base $ \G $ satisfied in all these interpretations, e.g., if $ \omega=\omega_5 $ would be equal to $ \omega_1,\omega_2 $ and $ \omega_3 $, and  $ \Mod{\G}=\{ \omega_1,\ldots,\omega_3,\omega_5 \} $.  However, this is not the case in $ \releqK^{\circ_\mathrm{Ex}} $ and such a  belief base $ \G $ does not exist in $ \mathbb{L}_\mathrm{Ex} $.
Therefore, the relation $ \releqK^{\circ_\mathrm{Ex}} $ is min-retractive.
\end{example}

The following Proposition summarizes what we have achieved so far by Example \ref{ex:logicEX} to Example \ref{ex:orderByEx2}.

\begin{proposition}\label{prop:non_tranisitive_by_example}
	There is a monotone logic $ \mathbb{L} $, a \basechange\ operator $ \circ $ on $ \mathbb{L} $ which satisfies (G1)--(G6), and a belief base $ \MC{K} $ such that $ \releqcK $ (cf. Definition \ref{def:relation_new}) is a min-friendly but not transitive relation.  
\end{proposition}

We continue with some remarks on min-retractivity and transitivity.

\subsection*{Remarks on Min-Retractivity and Transitivity}
In belief revision literature, transitivity and the postulates (G5) and (G6) are strongly associated.
Therefore, at a first glance, the introduction of the notion of min-retractivity seems to be artificial and unnatural in comparison to transitivity.
However, it turns out that min-retractivity is the property of capturing the nature of the AGM postulates (G5) and (G6) very exactly.

Suppose $ \mathbb{L} $ is a logic having a disjunction connective. 
It is well-known in belief revision that (G5) and (G6) are equivalent (in the light of (G1)--(G4)) to disjunctive factoring \cite{KS_Hansson1999} (here given in a semantic formulation for belief bases):
\begin{equation*}
	\Mod{\K \circ (\G_1 \lor \G_2)}  = \begin{cases}
		\Mod{\K \circ \G_1} & \text{ or }\\
		\Mod{\K \circ \G_2} & \text{ or }\\
		\Mod{\K \circ \G_1} \cup  \Mod{\K \circ \G_2} & 
	\end{cases}
\end{equation*}
Now consider the following observation on min-retractivity.
\begin{proposition}\label{lem:retartivity_trichotonomoty}
	Let $ \G_1,\ldots,\G_n,\G\subseteq \MC{P}_\mathrm{fin}(\MC{L}) $ be belief bases with $ \Mod{\G}=\Mod{\G_1}\cup \ldots \cup \Mod{\G_n} $ and let $ \preceq $ be a min-retractive relation on $ \Omega $. Then $ \min(\G,\preceq) = \bigcup_{i\in I}\min(\Mod{\G_i},\preceq) $ for some set $  I\subseteq \{1,\ldots,n\} $.
\end{proposition}
\begin{proof}
	Let $ \Mod{\G}=\Mod{\G_1}\cup \ldots \cup \Mod{\G_n} $ and $ \I\in \min(\Mod{\G},\preceq)  $.
	
	If $ \I\in \Mod{\G_j} $ then we have for each $ \I'\in {\min(\Mod{\G_j},\preceq)}  $ that $ \I' \preceq \I $ holds. 
	Since $ \preceq $ is min-retractive, we conclude that $ \I' \in  \min(\Mod{\G},\preceq) $. 
	Consequently, we obtain $ {\min(\Mod{\G_j},\preceq)} \subseteq {\min(\Mod{\G},\preceq)} $.
For the converse direction, let $ I $ be the smallest set such that $ {\min(\Mod{\G},\preceq)}\subseteq \bigcup_{i\in I} \Mod{\G_i} $. 
	Observe that from $\Mod{\G_i} \subseteq \Mod{\G}$ and $ \I\in{\min(\Mod{\G},\preceq)} $ the statement $ \I \in {\min(\Mod{\G_i},\preceq)} $ follows directly. 
	As a consequence, we obtain $ {\min(\Mod{\G},\preceq)}= \bigcup_{i\in I} {\min(\Mod{\G_i},\preceq)} $.\qedhere
\end{proof}

By Proposition~\ref{lem:retartivity_trichotonomoty}, when a relation $ \leq $ on the interpretations and an operator $ \circ $ are compatible, i.e. $ \Mod{\K \circ \G} = \min(\Mod{\G},\leq)  $, then min-retractivity of $ \leq $ guarantees disjunctive factoring.
This shows the strong and natural connection between the notion of min-retractivity and (G5) and (G6).

\subsection*{Expressibility and Satisfaction of (G1)--(G6)}

For our running example, we will now observe that $ \releqK^{\circ_\mathrm{Ex}} $ is also a min-expressible relation.

\begin{example}[continuation of Example \ref{ex:orderByEx2}]\label{ex:orderByEx3}
	Consider again $ \releqK^{\circ_\mathrm{Ex}} $, and observe that $ \releqK^{\circ_\mathrm{Ex}} $ is compatible with $ \circ_\mathrm{Ex} $, i.e. $ \Mod{\K\circ\G}=\min(\Mod{\G},\releqK^{\circ_\mathrm{Ex}}) $.
	
	Thus, for every belief base $ \G \in \MC{P}_\mathrm{fin}(\MC{L}_\mathrm{Ex}) $, the minimum $ \min(\G,\releqK^{\circ_\mathrm{Ex}}) $ yields a set expressible by a belief base.
	
	Theorem~\ref{thm:minexpressible} guarantees us that $ \circ_\mathrm{Ex} $ satisfies (G1)--(G6),  as we can extend $ \releqK^{\circ_\mathrm{Ex}} $ to a faithful min-expressible and min-friendly assignment.
\end{example}

\subsection*{Total Preorder Representability}

As last step, we will now employ the novel notion of \uncovered\ (cf. Definition \ref{def:uncovered_new}) and our representation theorem  (cf. Theorem \ref{thm:when_tranisitive}) for total preorder representability, to show that there is no (total) preorder assignment for the operator $ \circ_\mathrm{Ex} $ from our running example.

\begin{example}[continuation of Example \ref{ex:orderByEx3}]\label{ex:orderByEx4}
Consider again $ \mathbb{L}_\mathrm{Ex} $ from Example \ref{ex:logicEX}.
We will now see that $ \mathbb{L}_\mathrm{Ex} $ exhibits a critical loop.

For this, choose $ \G_i=\{ \varphi_{i+1} \} $ and $ \G'_i=\{\psi_{i+2} \} $  for $ i\in\{0,1,2\} $.
We consider each of the three conditions of Definition \ref{def:uncovered_new} in a separate case:

\smallskip
\noindent\emph{Condition (1).} Observe that $ \K $ from Example \ref{ex:orderByEx1} is inconsistent with $ \G_0 $, $ \G_1 $ and with $ \G_2 $. Thus, Condition (1) is satisfied.

\smallskip
\noindent\emph{Condition (2).} For each $ i\in\{0,1,2\} $, the belief base $ \G_i $ and\footnote{where $ \oplus $ is addition $ \mathrm{mod}\ 3 $.} $ \G_{i\oplus 1} $ are consistent, but $ \G_i \cup \G_{i\oplus 1} \cup $ is inconsistent with $ \G_{i\oplus 2} $, e.g. $ \Mod{\{\varphi_1\}}\cap \Mod{\{\varphi_2\}} =\{\omega_2\} $ and $\omega_2\notin \Mod{\{\varphi_3\}} $.
For satisfaction of Condition (2), observe that  $ \G'_i $ is equivalent to $ \G_i\cup\G_{i\oplus 1} $.

\smallskip
\noindent\emph{Condition (3).} The belief base $ \G=\{ \varphi_4 \} $ is the only belief base consistent with $ \G'_0  $, $ \G'_1 $, and $ \G'_2 $. 
For satisfaction of Condition (3) observe that $ \G'=\{\psi_4\} $ fulfils the required condition $ \emptyset\neq \Mod{\G'} \subseteq \Mod{\G}\setminus(\Mod{\G_0}\cup \Mod{\G_1} \cup \Mod{\G_2}) $.

In summary $ \G_0$,$\G_1$, and $\G_2$ form a critical loop for $ \mathbb{L}_\mathrm{Ex} $. 
Thus, by Theorem \ref{thm:when_tranisitive} $ \circ_\mathrm{Ex} $ is not total preorder representable, i.e., there is no min-friendly min-expressible preorder assignment compatible with $ \circ_\mathrm{Ex} $.

Moreover, observe that the construction of $ \circ_\mathrm{Ex} $ presented in Example \ref{ex:orderByEx1} illustrates the construction given in the proof of Proposition \ref{prop:when_tranisitiveIF}. 
In particular, for the example presented here one would obtain $ \mathfrak{B}'=\{ \{\varphi_4 \} \} $ when following the outline of the proof.
\end{example}

The following proposition summarizes an implication of the running example we presented here.

\begin{proposition}\label{prop:non_ciritical_loop_existence}
	There is a monotone logic $ \mathbb{L} $, a \basechange\ operator $ \circ $ on $ \mathbb{L} $ which satisfies (G1)--(G6) and a belief base $ \MC{K} $ such that $ \circ $ is a not total preorder representable.  
\end{proposition}